\documentclass{article}
\usepackage{microtype}
\usepackage{graphicx}
\usepackage{subfigure}
\usepackage{booktabs} % for professional tables

\usepackage{svg}
\usepackage{tikz}
\usepackage{amsmath}

\usepackage{amssymb}
\usepackage{mathtools}
\usepackage{amsthm}
\usepackage{pythonhighlight}
\usepackage[toc,page]{appendix}
\newcommand{\tabincell}[2]{\begin{tabular}{@{}#1@{}}#2\end{tabular}}

\theoremstyle{plain}
\newtheorem{theorem}{Theorem}[section]

\newtheorem{lemma}[theorem]{Lemma}

\theoremstyle{definition}
\newtheorem{definition}[theorem]{Definition}

\theoremstyle{remark}

\usepackage{hyperref}

\usepackage[accepted]{icml2022}

\icmltitlerunning{What Dense Graph Do You Need for Self-Attention?}

\begin{document}

\twocolumn[
% \icmltitle{A Graph Scoring function for graph-base attentions and thinking hypercube \\ }
% \icmltitle{How does sparsity affect performance of Transformers ? Sparse bottleneck   \\ }
\icmltitle{What Dense Graph Do You Need for Self-Attention?   \\ }
% \icmltitle{Rethinking Sparse Transformer with Tesseract Graph \\ }
% \icmltitle{Tesseract Transformer: In search for optimal composition for information transfer\\}

\begin{icmlauthorlist}
\icmlauthor{Yuxin Wang}{fudan,ling}
\icmlauthor{Chu-Tak Lee}{fudan}
\icmlauthor{Qipeng Guo}{fudan}
\icmlauthor{Zhangyue Yin}{fudan}
\icmlauthor{Yunhua Zhou}{fudan} \\
\icmlauthor{Xuanjing Huang}{fudan,ling}
\icmlauthor{Xipeng Qiu}{fudan,pengcheng}
\end{icmlauthorlist}

\icmlaffiliation{fudan}{School of Computer Science, Fudan University}
% \icmlaffiliation{fudan}{Fudan University}
\icmlaffiliation{pengcheng}{Peng Cheng Laboratory}
\icmlaffiliation{ling}{Institute of Modern Languages and Linguistics, Fudan University}

% \icmlaffiliation{aws}{AWS AI, Shanghai, China.}

\icmlcorrespondingauthor{Yuxin Wang}{wangyuxin21@m.fudan.edu.cn}
\icmlcorrespondingauthor{Xipeng Qiu}{xpqiu@fudan.edu.cn}

\icmlkeywords{Sparse Attention, Transformer, Block sparsity, ICML}

\vskip 0.3in
]



\printAffiliationsAndNotice{}  % leave blank if no need to mention equal contribution

\begin{abstract}

Transformers have made progress in miscellaneous tasks, but suffer from quadratic computational and memory complexities. Recent works propose sparse Transformers with attention on sparse graphs to reduce complexity and remain strong performance. While effective, the crucial parts of how dense a graph needs to be to perform well are not fully explored. In this paper, we propose Normalized Information Payload (NIP), a graph scoring function measuring information transfer on graph, which provides an analysis tool for trade-offs between performance and complexity. Guided by this theoretical analysis, we present Hypercube Transformer, a sparse Transformer that models token interactions in a hypercube and shows comparable or even better results with vanilla Transformer while yielding $O(N\log N)$ complexity with sequence length $N$. Experiments on tasks requiring various sequence lengths lay validation for our graph function well\footnote{Code is available at https://github.com/yxzwang/Normalized-Information-Payload.}.

\end{abstract}

\section{Introduction}
In recent years, self-attention and its implementation Transformers~\cite{vaswani2017attention} have achieved great success in a wide variety of Natural Language Processing (NLP)~\cite{devlin2019bert,vaswani2017attention,miller2019leveraging,sun2019utilizing} and Computer Vision (CV)~\cite{yuan2021tokenstotoken,dosovitskiy2021image} tasks. 

The key innovation of self-attention mechanism is to allow each token to interact with others directly, and thus avoid the long-term dependency problem. However, this results in the quadratic computational and memory complexity with the sequence length. To improve model efficiency, many lightweight Transformers are proposed ~\cite{tay2020efficient,lin2021survey}. Among them, 
sparse Transformers~\cite{zaheer2021big,beltagy2020longformer,child2019generating,guo2019starTransformer} utilize sparse attention in the self-attention mechanism including global attention, window attention or rule-based sparse attention.

\begin{figure}[t]
\centering
\includegraphics[scale=0.2]{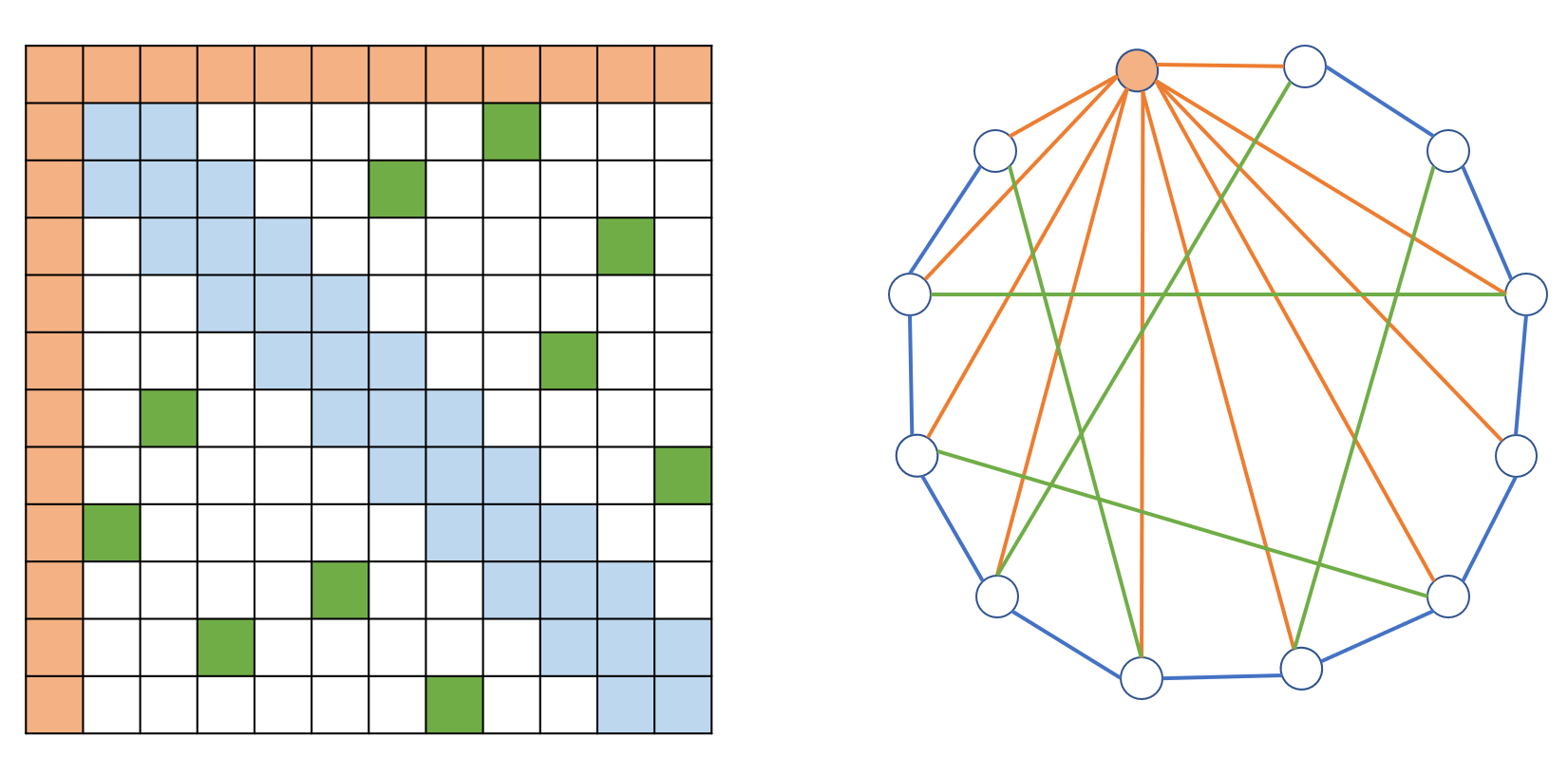}

\caption{Attention map and its corresponding graph.}
\label{fig:sparse_attn}

\end{figure}

Previous works view sparse attention as self-attention on sparse graphs. Figure \ref{fig:sparse_attn} shows typical sparse attention map and its corresponding graph. Vanilla self-attention can be regarded as a complete graph. Although these Sparse Transformers have made progress, there still remains some questions: which property is important for those graphs serving as ground for self-attention? How dense do we need the graph to be in order to reduce complexity and at the same time remain performance? While some~\cite{zaheer2021big,chen2021pixelated} showed theoretical analysis for existing sparse self-attention, they do not provide a general analysis tool for comparing different sparse patterns.

 To investigate further into differences between sparse patterns, we need theoretical analysis based on sparse graphs. In this paper, we propose Normalized Information Payload (NIP), a graph scoring function to measure information transfer on a given graph. Our function can also be applied to analyze previous sparse patterns and provide insights for their empirical results. Guided with our proposed function, we further present Hypercube Transformer which adapts hypercube into self-attention and achieves great balance between performance and computation costs. Experiments in long-context sequence modeling and large-corpus pretraining show that Hypercube Transformer is competitive both theoretically and practically.

The contributions of our paper can be summarized as follows:

\begin{itemize}
\item We propose Normalized Information Payload (NIP), a graph scoring function that tries to investigate important properties of graph used in self-attention. This function provides theoretical analysis tools for performance and complexity of different graphs used in position-based sparse self-attention.

\item Guided with our theoretical analysis, we present Hypercube Transformer, which can behave competitively compared to vanilla Transformer in various tasks and better in long context tasks while requiring less time and computation.

\end{itemize}

\section{Graph Scoring Function for Balancing Costs and Performance }

\textbf{What Dense Graph Do We Need?} Self-attention (vanilla and sparse) can be viewed as attention-based information transfer on graphs. While many sparse patterns based on sparse graphs have been presented, there still remains questions: what is important for a sparse graph in self-attention? What dense graph do we need for self-attention? How dense is optimal for the graph to balance cost and performance? To answer these questions, we propose Normalized Information Payload, a graph scoring function to score any graphs used for self-attention.

\subsection{Normalized Information Payload}

\begin{table*}[!h]
\centering
\small
\caption{\footnotesize Normalized Information Payload for commonly used graphs, where $w$ is the number of neighbors in ring lattice. $\star$: $\Theta\left(\frac{1}{N^2}\right)$ after refinement.}
\vskip 0.15in
\begin{tabular}{@{}lccccccccc@{}}
\toprule
Type of graph & $\text{CC}(G)\downarrow$ & $\text{\text{IP}}(G)\uparrow$                           & $\text{NIP}(G)\uparrow$                                       \\
\midrule
Complete      & $\Theta(N)$              & $\Theta\left(\frac{1}{N}\right)$                                    & $\Theta\left(\frac{1}{N^2}\right)$                                              \\
% Ring lattice       & $\Theta(N)$              & $\Theta(\frac{N}{w^N})$                          & $\Theta(\frac{1}{w^N})$                                      \\
E-R random    &$\Theta(\log^2 N)$      & $\Theta\left(\frac{(N-2)!}{N^{\log N}(N-\log N)!}\right)$ & $\Theta\left(\frac{(N-2)!/(N-\log N)!}{N^{\log N}\log^2(N)}\right)$   \\
Tree          & $\Theta(\log N)$        & $\Theta(\frac{1}{N^{\log(9)}})$                   & $\Theta\left(\frac{1}{N^{\log(9)}\log N}\right)$                       \\
Star          & $\Theta(1)$              & $\Theta\left(\frac{1}{N}\right)$                                    & $\Theta\left(\frac{1}{N}\right)^{\star}$                                               \\
\midrule

\tabincell{l}{Ring lattice\\ \quad + E-R random }   & $\Theta(\log N(\log N+w))$        & $\Theta\left(\frac{(N-2)!}{(N+\frac{w}{\log N})^{\log N}(N-\log N)!}\right)$                   &  $\Theta\left(\frac{(N-2)!/(N-\log N)!}{(N+\frac{w}{\log N})^{\log N}\log N(\log N+w)}\right)$                       \\

\tabincell{l}{Ring lattice\\
\quad + Star (Longformer) }         & $\Theta(w)$              & $\Theta\left(\frac{1}{Nw}\right)$                                     & $\Theta\left(\frac{1}{Nw^{2}}\right)$                                            \\
\tabincell{l}{Ring lattice\\
\quad + Star\\
\quad + E-R random (BigBird)}         & $\Theta(\log N+w)$              & $\Theta\left(\frac{1}{N(\log N+w)}\right)$                                 & $\Theta\left(\frac{1}{N(\log N+w)^2}\right)$                                               \\
\midrule
Hypercube     &$\Theta(\log^2 N)$     & $\Theta\left(\frac{(\log N)!}{(\log N)^{\log N}}\right)$         & $\Theta\left(\frac{(\log N)!}{(\log N)^{\log N+2}}\right)$           \\
\bottomrule
\end{tabular}
\label{tab:Gscore}
\end{table*}
Generally, we expect our model to grab all interactions among tokens. Sparse attentions indirectly capture these interactions by multi-layer information transfer. That comes to two questions: what costs do we pay for grabbing these interactions and how much information can be transferred? To answer these questions respectively, we consider graphs in two aspects: Computational Complexity and Information Payload. 

\textbf{Computational Complexity.} Computational Complexity ($\text{CC}$) of a graph $G$, denoted by $\text{CC}(G)$, is the computation complexity required to allow the model to grab all interactions among tokens when using graph $G$ for self-attention. The requirement for graph $G$ used here is that $G$ is connected. Lower Computational Complexity of a graph makes the whole model less computational expensive. 

\textbf{Information Payload.} The amount of information transfer is also important. For instance, sequence models like LSTMs~\cite{HochSchm97} is able to transfer information in a long sequence but suffers from long-term dependency and low information capacity. So we introduce Information Payload ($\text{IP}$) for a graph $G$, denoted by $\text{IP}(G)$, measuring how much information a graph can transfer when allowing the model to grab all interactions among tokens. Higher Information Payload of a graph enables the whole model to grab more information.

To better compare information transfer on different graphs, we define the Normalized Information Payload for a graph $G$, denoted by $\text{NIP}(G)$, as follows,
\begin{equation}
    \text{NIP}(G):=\frac{\text{IP}(G)}{\text{CC}(G)}. 
\end{equation}
The higher $\text{IP}(G)$ is, the more information a graph can transfer. The lower $\text{CC}(G)$ is, the less computational resources the graph costs, which also means that the graph can be used to model long sequences. And the higher score of $\text{NIP}(G)$ implies the graph can perform well in real-world tasks. This is consistent with our observations in experiments. Then we will demonstrate how these two components of our function are defined in detail and show how previous works like BigBird fit our function. For a self-attention layer on graph $G$, we denote it by $G$-attention layer briefly.

\subsubsection{Computational Complexity }
Computational Complexity is related to grabbing all interactions among tokens. Given a $G$-attention layer, to make the whole model grab all interactions among tokens, we need to stack $ \kappa(G)$ $G$-attention layers. Straightforwardly, $\kappa(G)$ is the diameter of graph $G$. And for one $G$-attention layer, the Computational Complexity is proportional to the total number of edges in $G$. When the input sequence is fixed at length $N$, the Computational Complexity for one layer is proportional to the mean degree of $G$, which we denoted by $\rho(G)$ here. Intuitively $\rho(G)$ measures the complexity of the graph itself and $\kappa(G)$ is how many times we forward the graph. 

\begin{definition}
Let Computational Complexity $\text{CC}$ for a given graph $G$ be 
\begin{equation}
    \text{CC}(G):= \rho(G) \times \kappa(G).
\end{equation}
\end{definition}

$\text{CC}(G)$ shows the minimum computation costs to ensure information exchange for every node pair in a graph.

\begin{figure}

  \centering 
  \includegraphics[scale=0.32]{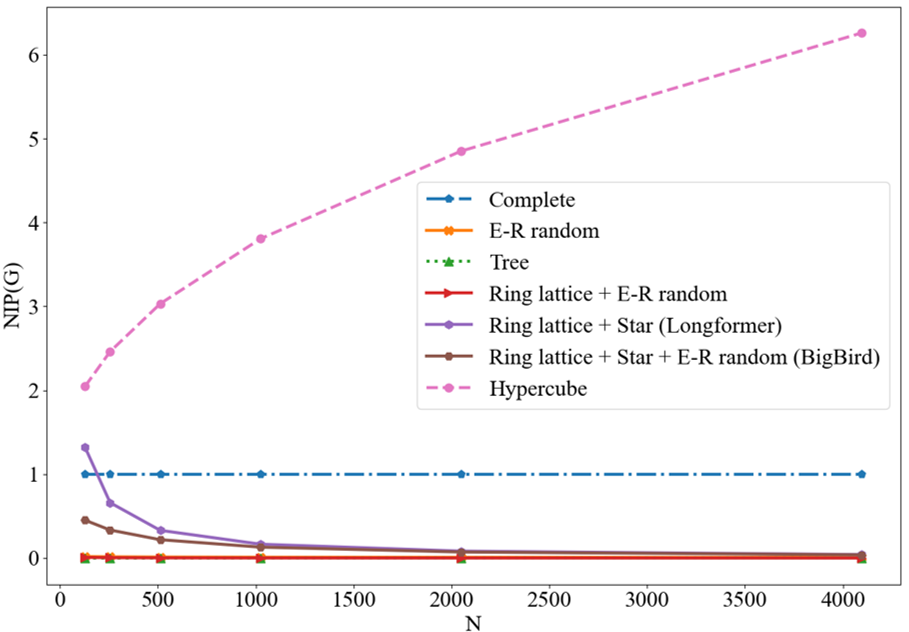} %
  \caption{$\text{NIP}(G)$ for graphs divided by complete graph in Table \ref{tab:Gscore}. We do not include star graph and ring lattice in this Figure because $\text{NIP}(G)$ for star graph is too large. The $w$ used for ring lattice is set to $\frac{N}{16}$ according to Longformer at length $4096$.}\label{fig:IT}
\end{figure}
\subsubsection{Information Payload}
To capture all interactions among tokens is not enough, we have to take into account how much information a graph can transfer after that. Information Payload is introduced to measure it. First we examine how information transfer is like in self-attention. Since we have Softmax operation in self-attention mechanism acting as normalization, it is straightforward to set the total amount of information a node can receive as one unit of information. For one node $i$ with degree $deg(i)$, the average information it can receive is $\frac{1}{deg(i)}$, so the total Information Payload for one path is related to all end nodes on the path.

\textbf{Notation.}
Given a graph $G(\mathcal{V,E})$, and nodes $a, b$, let $\mathcal{P}_{ab} $ be the set of all paths that start with node $a$ and end with node $b$ and have length equal to the distance between node $a$ and node $b$. For one path $P_{ab} \in  \mathcal{P}_{ab} $, $len(P_{ab})$ is the length of $P_{ab}$. For one node $i$, we use $deg(i)$ to denote the degree of it.

%\textbf{Our definition.}
\begin{definition}
\label{def:setofpaths}
For one path $P_{ab} \in \mathcal{P}_{ab} $, the Information Payload of one path $P_{ab}$, denoted by $R(P_{ab})$ , is defined as 
\begin{equation}
    R(P_{ab}):=\prod_{v \in P_{ab} \And v \neq a}\frac{1}{deg(v)}.
\end{equation}

\end{definition} 
Next we define the Information Payload between node $a$ and node $b$.
\begin{definition}
The Information Payload between node pair $(a,b)$, denoted by $I_{ab}$ is sum of Information Payload of all paths that belong to $\mathcal{P}_{ab} $ :

\begin{equation}
    I_{ab}=\sum_{P_{ab} \in  \mathcal{P}_{ab} }R(P_{ab}).
\end{equation}
\end{definition}

Note that $\frac{1}{k_i}$ is equal to the probability that a random walk starts from the node itself to any of its neighbors, which makes our Information Payload closely related to random walk on graphs.

\textbf{Relationship between Information Payload and Random Walk.} Our Information Payload is deduced from self-attention forward, which is closely related to random walk. We show in Figure \ref{pic:attention_randomwalk} that self-attention is like reversed random walk. In Figure \ref{pic:attention_randomwalk}, each column shows a $G$-attention layer and we have $t$ layers. For $G$ used here, we present three nodes $a,b,c$ and edges $(a,b)$, $(a,c)$ and self-loop of three nodes. Figure \ref{pic:attention_randomwalk}(a) shows the updating of $G$-attention layers. The red paths show the information transferred from node $a$ to node $c$. In Figure \ref{pic:attention_randomwalk}(b), blue paths show the random walk starts from node $c$ to node $a$. We can see that the Information Payload from node $a$ to node $c$ after $t$ layers is equal to the probability of a random walk starting from node $c$ ends in node $a$ at the $t$th step. We demonstrate it in the theorem below in detail.

\begin{figure}[h]

\centering
\subfigure[Attention forward]{
\includegraphics[scale=0.2]{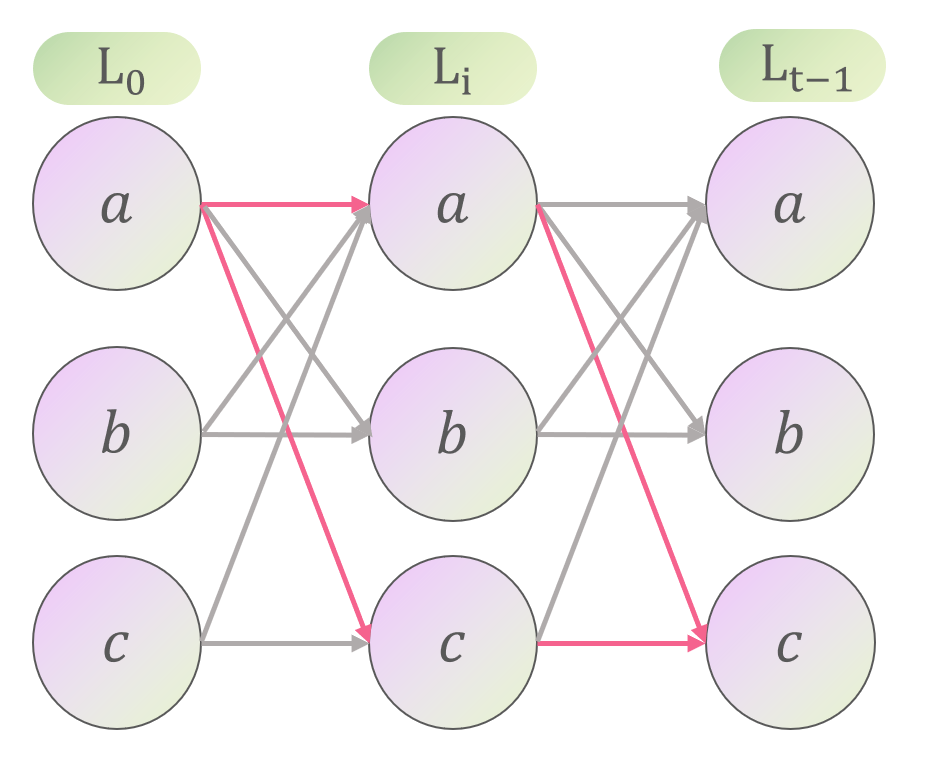}
}\quad
\quad
\subfigure[Random walk]{
\includegraphics[scale=0.2]{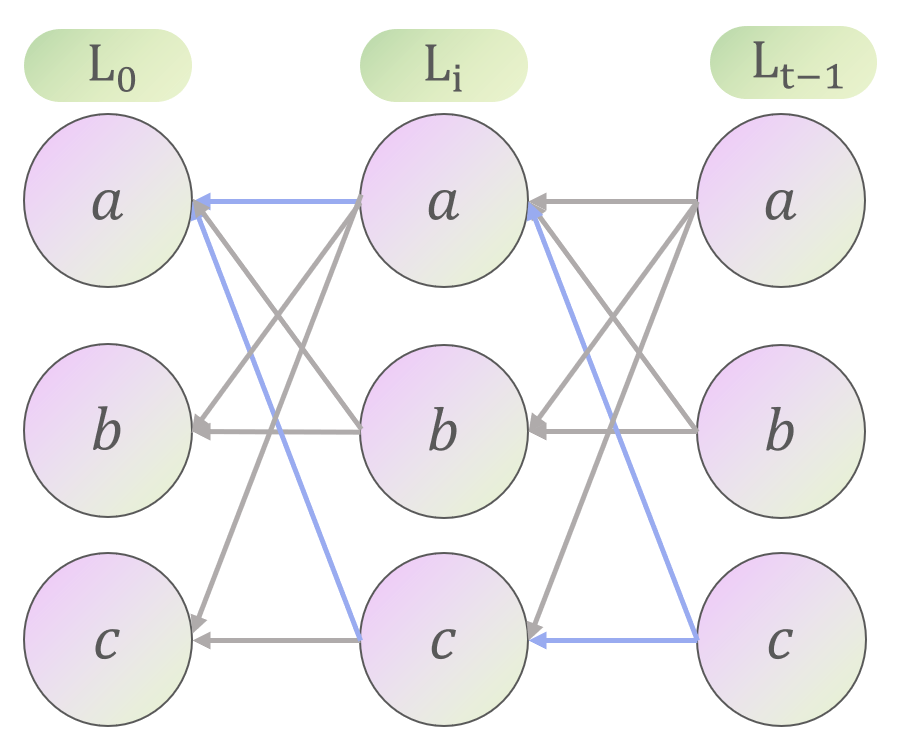}
}\quad

\caption{Relationship between $G$-attention layer and random walk. Red lines in Figure (a) shows attention forward from node $a$ to node $c$ across $t$ layers while blue lines in Figure (b) shows random walk starting from node $c$ to node $a$.}\label{pic:attention_randomwalk}

\end{figure}

\begin{theorem}
\label{Iab}
Information Payload between two nodes $I_{ab}$ equals to the probability of a random walk starts from node $b$ that ends in node $a$ at step $len(P_{ab})$.
\end{theorem}

Proof of theorem \ref{Iab} and calculating $I_{ab}$ via random walk is in Appendix \ref{proof:Iab}.

Note that the probability of a random walk starts from $b$ that ends in $a$ within $len(P_{ab})$ steps equals to zero. Thus $I_{ab}$ measures the amount of information when information flow first reaches node $a$ from node $b$. 

Since we have defined the Information Payload between two nodes, the Information Payload for the whole graph is chosen to the Information Payload between node pairs $(a,b)$ whose distance is the diameter of the graph.

\begin{definition}
\label{def:FG}
The Information Payload for a graph $G$ $\text{IP}(G)$ is the smallest Information Payload $I_{ab}$ between node pairs $(a,b)$ whose distance is the diameter of the graph. Let $\Delta$ be the set of node pairs whose distance is the diameter of the graph, we have 
\begin{equation}
    \text{IP}(G):=\min_{(a,b) \in \Delta}{I_{ab}}.
\end{equation}
We consider the minimum Information Payload for node pair $(a,b)$ whose distance is the diameter, so we choose the smallest value to guarantee the lower bound of Information Payload, which is motivated by the theory of the Cannikin Law~\cite{e21090863}, that the capacity of the wooden bucket is limited by the height of its shortest plank.
\end{definition}

\subsubsection{Discussion}

Table \ref{tab:Gscore} lists $\text{CC}(G)$, $\text{IP}(G)$ and $\text{NIP}(G)$ for commonly used graphs\footnote{In this paper, all $\log$ means $\log_2$.}. The detailed computation is in Appendix \ref{appendix:computation}. 

\textbf{Self-loop.} In self-attention, a token can always attend to itself, meaning that all graphs have self-loop. $\text{IP}(G)$ use node pair whose distance is the diameter of the graph, so self-loop does not affect its value much. Also $\text{CC}(G)$ does not change much even if we have a large number of nodes. Thus $\text{NIP}(G)$ does not change much compared to graphs without self-loop. 

\textbf{Markov Chains.} In self-attention, the information updating for a node is not uniform with its neighbors but dependent on representations of all relative nodes, which is not the case of normal random walk. However, previous analysis~\cite{https://doi.org/10.48550/arxiv.1906.04341} on BERT attention shows that some attention heads, especially in lower layers, have very \textbf{broad attention}. Since information in lower layers is closer to the input and is important, we set the uniform distribution of attention weights and this setting is also the situation of random initialization. Also, viewing random walk as Markov chains can bring insights into this situation. If we let the transition matrix in Markov chains be attention-based, then normal random walk becomes attention-based random walk, which is suitable for self-attention. 

\subsection{Case study}

To make $\text{NIP}(G)$ clear, we show in Figure \ref{fig:IT} all the $\text{NIP}(G)$ in Table \ref{tab:Gscore} except star graph. Because $\text{NIP}(G)$ for star graph is too large.

\begin{figure*}
  \centering 
  \includegraphics[scale=0.45]{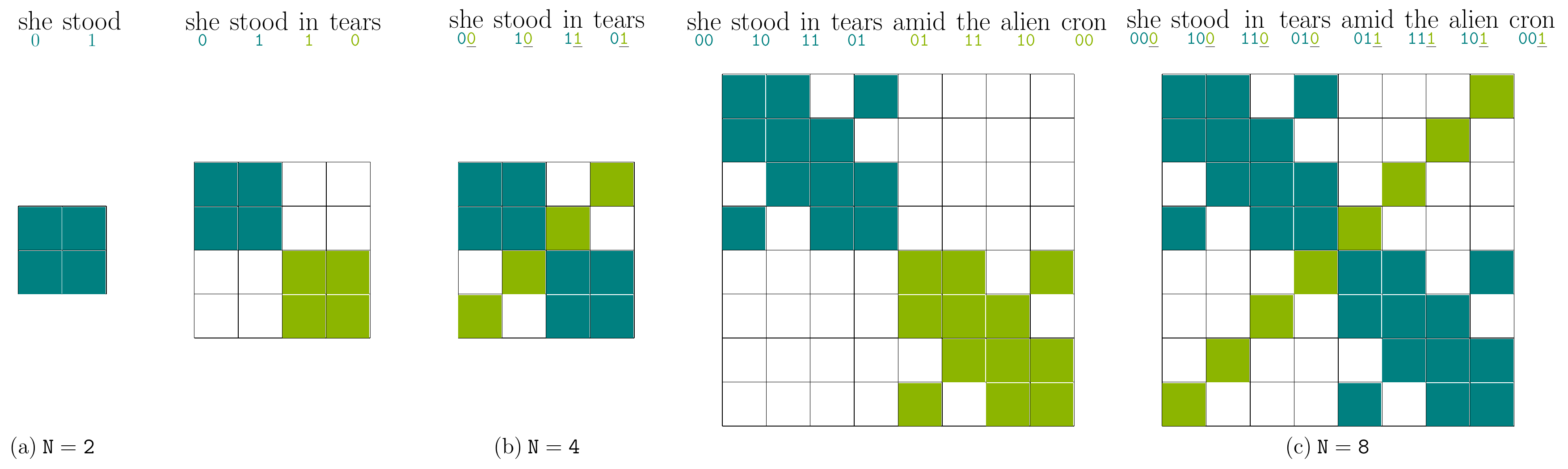} %
  \caption{Iteratively mapping a sequence to a hypercube and its attention mask. Figure (a), (b) and (c) is the attention map for input sequences with length $N=2,4,8$ respectively.}
  \label{fig:mapping}
\end{figure*}

\textbf{Star graph.}

Star graph has the highest $\text{NIP}(G)$, however, huge amounts of information flow through the global node can cause a bottleneck of information transfer. This bottleneck reduces the Information Payload of star graph to $\frac{1}{N-1}$ of the original one because $N-1$ local nodes versus one global node. Thus, the refined $\text{NIP}(G)$ of star graph is $\Theta\left(\frac{1}{N^2}\right)$.

\textbf{Ring lattice.}
 Ring lattice is often used in self-attention known as window attention or local attention. It is usually combined with other attention patterns so we compute NIP for its mixtures. Most of attention occupies in neighborhoods~\cite{cui-etal-2019-fine}, which makes previous sparse patterns adapting it reasonable. 

\textbf{Random graphs.}
We compute the expected values for random graphs in this paper. For the Erdős–Rényi (E-R) random graph used in BigBird, if the probability of every edge to exist is $p$, the E-R random graph is highly possibly connected if $p>\frac{(1+\epsilon)\ln(N)}{N}$. This means that to make connected random graphs with high probability, the average degree of the graph is more than $\frac{(1+\epsilon)\ln(N)(N-1)}{2N}$. We use the lower bound and set $p$ to $\Theta\left(\frac{\log N}{N}\right)$ here. We limit the choice of random graphs in this paper to E-R random graph because it has been applied in self-attention.

\textbf{BigBird, Longformer and mixed graphs.}
BigBird is mixed from star graph, ring lattice and random graph, while Longformer is only mixed from star graph and ring lattice. For BigBird, because of random graph, we compute the expected $\text{NIP}(G)$ for a given graph $G$. We use the settings for random graph as we mentioned in above paragraph. We have shown that adding random graph to Longformer like BigBird not necessarily improves the Normalized Information Payload for the graph because it makes the graph more complex and reduces the expectation Information Payload. While such observation is counterintuitive, the BigBird-ETC which is sota in BigBird models also despose random attention. Our ablation studies in the experiment section (in Table \ref{tab:lra}) also demonstrate that mixing Longformer with random attention does not necessarily improve the performance.

\section{Hypercube Transformer}
Guided with our Normalized Information Payload, we aim to find better graphs for self-attention and present Hypercube Transformer. Like vanilla Transformer, our Hypercube Transformer utilize positional embeddings to catch positional information.

\subsection{Hypercube}

\begin{figure}
  \centering 
  \includegraphics[scale=0.4]{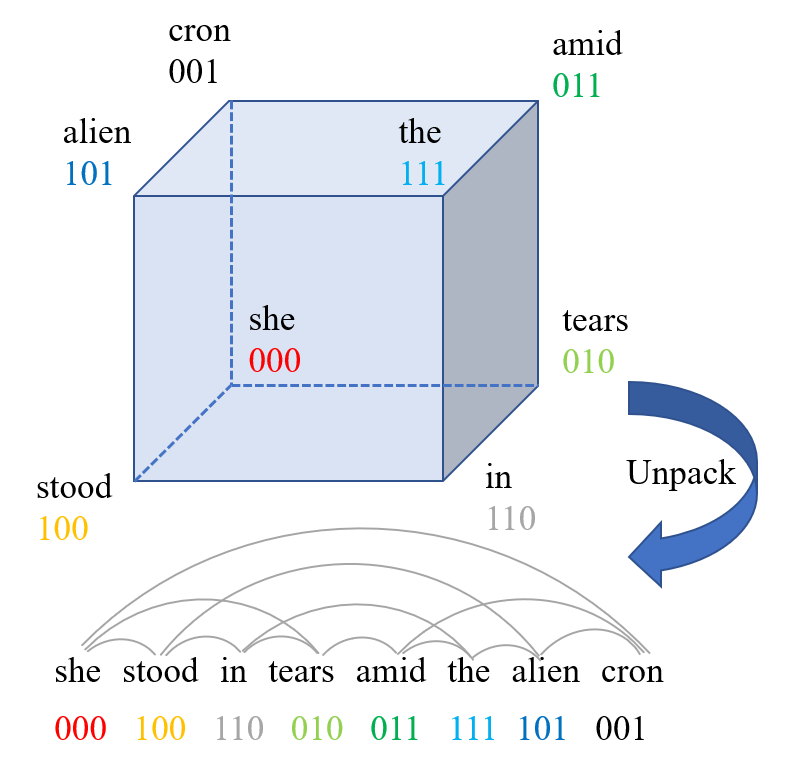} %
  \caption{Unpacking a hypercube to a sequence. Tokens that are neighbors in hypercube are also neighbors in a sequence.}
  \label{fig:unpacking}
\end{figure}

The global node in star graph is very important but suffers from information interference as mentioned before. Hence we search for regular graphs with no inductive bias to alleviate information interference. An observation for complete graph is that the shortest distance between two neighbors' neighbors (excluding two nodes themselves) equals to zero, meaning that every two neighbor has the same neighbor. We propose a graph in which this distance equals to one, which makes the whole graph much sparser while maintaining the connectivity of the graph, and that graph is hypercube.

% The global node in star graph is very important but suffers from information interference as mentioned before. Hence we search for regular graphs with no inductive bias to alleviate information interference. An observation for complete graph is that the shortest distance between two neighbors' neighbors (excluding two nodes themselves) equals to zero, meaning that every two neighbor has the same neighbor. We contend that what if this distance becomes one, which makes the whole graph much sparser while maintaining the connectivity of the graph, and that comes to hypercube.

After our computation, hypercube is potential according to our Normalized Information Payload. It has Normalized Information Payload larger than other sparse graphs except star graph as shown in Figure \ref{fig:IT}. At the same time, small $\text{CC}(G)$ means hypercube can be applied to long sequences. Also, in parallel computing, hypercube is one of the most useful communication structures which encounters not much communication interference.

\subsection{Mapping Sequences to Hypercube\protect\footnote{We thank Anjiang Wei at Stanford University for his help.}}

Although we show that hypercube is potential in Normalized Information Payload, it's hard to apply it to real-world tasks if we can't map sequences to hypercube. Generally, we suggest a mapping method have two good properties:maintaining the original neighborhoods in sequences and easy to be extended to longer sequences. One example of unpacking hypercube to sequence is shown in Figure \ref{fig:unpacking}. Here we propose a novel iterative binary-number-based mapping from sequence to hypercube. We show this iterative binary mapping algorithm in Appendix \ref{alg:mapping}. We also show the iterative mapping pipeline and its attention map in Figure \ref{fig:mapping}.

  If the dimension of binary numbers is $k$, the final binary representation for token with index $i$ ($X_i:=X_i^{k-1}X_i^{k-2}...X_i^{1}X_i^{0}$, $i \in [0,N-1]$) is given by 
  \begin{equation}
      X_i^d = \left(\lfloor \frac{i \mod 2^{k-d+1}} {2^{k-d}}\rfloor + \lfloor \frac{i \mod 2^{k-d}}{2^{k-d-1}}\rfloor\right)\mod 2,
  \end{equation}
  where $\lfloor x\rfloor$ means the largest integer smaller than $x$. We denote the set of representation that has only one digit different from $X_i$ by $\mathcal{N}(X_i)$. After we got $X_i$ for every token with index $i$, we can easily draw the attention map by the rule that one token with representation $X_i$ can only attend to tokens with representation $X_j \in \mathcal{N}(X_i)$.

\begin{table*}[ht]
\centering
\small
\caption{Performances for different graphs on Long-Range Arena. $\star$ means after refinement.}
\vskip 0.15in
\setlength{\tabcolsep}{.8mm}
\begin{tabular}{@{}lccccccccccc@{}}
\toprule
Graph                             & \bf ListOps  & \bf Text  & \bf Retrieval  & \bf Image  & \bf Pathfinder   & \bf Avg. & \bf $\text{NIP}(G)$ & \bf SpeedUp     \\
\#Length                           & 2K           & 4K        & 4K             & 1K         & 1K               &    &       &     \\
\midrule
Complete                     &  37.20            & 63.54         & 81.00              & 47.23         & \textbf{74.39}                & 60.67    & $1\times$  &  $1\times$    \\
Star                             & 37.58         & 63.37     & 79.71          & 52.19     & 66.92                & 59.95            & $1\times^{\star}$ & -  \\
\tabincell{l}{Ring lattice + E-R random }                    & 36.44           & \textbf{63.81}         & 80.17              & 50.88          & 67.14                & 59.69         &  $1.43e^{-10}\times$  &- \\
\tabincell{l}{Ring lattice + Star (Longformer)  }      & 37.55           & 61.12     & 80.53         & 52.13      & 68.66                & 60.00          &   $8.25e^{-2}\times$ &- \\
\tabincell{l}{Ring lattice + Star + E-R random  (BigBird)}  & \textbf{37.80}        & 62.34     & 79.49          & 52.87      & 67.44            & 59.99        &  $7.21e^{-2}\times$  & - \\
\midrule
\textbf{Hypercube}                 & 37.48         & 63.79     & \textbf{81.16}          & \textbf{53.79}      & 74.12            & \textbf{62.07}& $4.85\times$ &  $15.8\times$\\
\bottomrule
\end{tabular}
\label{tab:lra}
\end{table*}

\subsection{Block Sparsity}
Block sparse pattern is introduced by~\cite{gray2017gpu} to tackle hardware problems for efficient calculating. It splits the sequence into several blocks and impose sparse patterns on blocks instead of tokens. Tokens can attend to every token in the same block and the block its block can attend to. Adapting block sparse lowers the sparsity of graph and intensifies information passing when the number of $G$-attention layer is fixed. While all sparse patterns adapt block sparse, we don't compare those patterns after adapting block sparse. However, we consider how block sparse have effects on sparse patterns and we have the theorem below.
\begin{theorem}
\label{blocksparse}
For block size $b\leq \frac{N}{2}$, larger block size makes star graph and hypercube have less Normalized Information Payload. 
\end{theorem}
We put the proof in Appendix \ref{proof:blocksparse}. Although block sparse reduces Normalized Information Payload in most situation, we still adapt it in our experiment section for the sake of training efficiency.

\section{Experiments}
Experiments are conducted to validate theoretical results of Normalized Information Payload and then show performances of Hypercube Transformer. All experiments are conducted on RTX 3090 GPU. In detail, We implement all Sparse Transformers with self-loop and block sparse using triton~\cite{10.1145/3315508.3329973}. We choose to conduct experiments mainly on block size 16 which is the smallest size allowed for triton. Recent work~\cite{guo2021longt5} also shows that block size 16 is cost effective on long range text tasks. 

\subsection{$\text{NIP}(G)$ and Performance on LRA}
To validate Normalized Information Payload, we focus on how our Normalized Information Payload is related to real-world situations.
\label{subsection:ablationforgraphs}
We rely on Long Range Arena (LRA) benchmark~\cite{tay2020long} for validation for our graph scoring function, $\text{NIP}(G)$. Long Range Arena is a benchmark testing how well a model can capture long range dependencies with tasks with input lengths varied from 1024 to 4096. By deliberately making the task harder, such as training text on byte level and image on pixel level, LRA serves as a systematic and popular proxy for performance to computational efficiency. In our case, LRA restricts models to have equal or less than four layers, making it a good test bed for layer efficiency. We choose several different graphs and apply them on $G$-attention layers in Transformer to evaluate performance. 

\textbf{Implementation details.} For the sake of consistency and simplicity, we do not strictly follow setting in the official implementation. Instead we use a four-layer network, in which every two layers with shared parameters for all tasks. All hyper-parameters are listed in Table \ref{tab:lra_hparams} in Appendix \ref{appendix:hyperparams}. Empirically we find that our architecture has fewer parameters than the original paper. We follow optimization configuration in~\cite{ma2021luna} and run experiments for different sparse graphs listed in Table \ref{tab:Gscore}. Results for our experiments and Normalized Information Payload according to Table \ref{tab:Gscore} with $N=2048$ are in Table \ref{tab:lra}. We run each experiment for three times with different random seeds and report the average accuracy.

\begin{figure}[h]
  \centering 
  \includegraphics[scale=0.25]{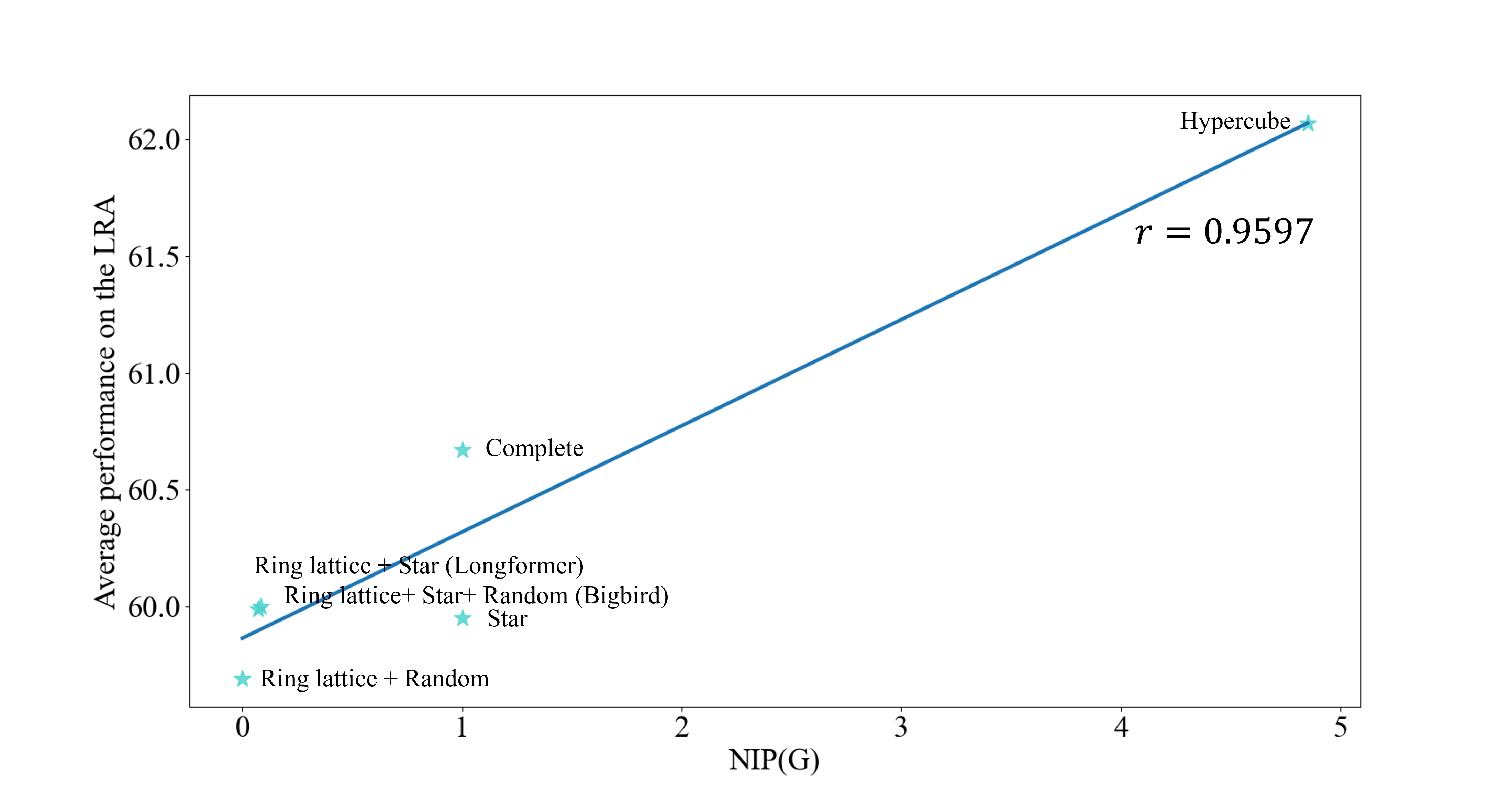} %
  \caption{Average performance on the LRA benchmark can have strong proposition with our proposed Normalized Information Payload.}
  \label{fig:IT-performance}
\end{figure}

\textbf{Results.} In Table \ref{tab:lra}, we observe that Hypercube Transformer outperforms all other graphs including complete graph in average. Apart from star graph, which suffers from information interference in the global node and results in learning problems, other graphs' performance fits Normalized Information Payload well. We draw the average performance of various graphs on the LRA benchmark to $\text{NIP}(G)$ in Figure \ref{fig:IT-performance} for these graphs. The average performance is strongly proportional to Normalized Information Payload in that the Pearson Correlation Coefficient between two variables is equal to \textbf{0.96}, which lays validation for our graph scoring function well. 

\textbf{Speedup.} For speedup, we only report the training speedup of hypercube compared to complete graph at 4096 length. We do not compare hypercube with other sparse patterns because to make fair comparison between BigBird pattern and hypercube, we choose the number of blocks for each pattern to be the same. Therefore, the speedup of two sparse patterns are close. Detailed block numbers are listed in Appendix \ref{appendix:hyperparams}.

\begin{table}[!h]
\centering
\small
\caption{Performance of Hypercube Transformer with different block sizes.}
\vskip 0.15in
\begin{tabular}{@{}lccccccccc@{}}
\toprule
\bf Hypercube     & \bf Retrieval & \bf Image \\
\midrule
Block size 16 & 81.16         & 53.79     \\  %  80.98,81.55,80.94 & 53.91,53.73,53.74 \\% seed 45,53,54
Block size 32 & 80.74         & 51.98     \\  %  80.58, 80.97, 80.69 & 51.82, 51.92, 52.2 \\ %seed 42, 43, 44
Block size 64 & 80.75         & 50.75     \\  % 80.45, 80.83, 80.99 & 51.33, 50.66, 50.27  \\ %seed 42, 43, 44
\bottomrule
\end{tabular}
\label{tab:blocks}
\end{table}
\textbf{Block size impact.} From Table \ref{tab:lra}, we find that Retrieval and Image can differentiate graphs better among five sub-tasks. So to validate theory \ref{blocksparse}, we investigate how block size affects performance empirically on these two sub-tasks of LRA. In Table \ref{tab:blocks}, we find that larger the block size, lower the performance is, which is corresponding to our theory \ref{blocksparse} that larger block size will result in lower Normalized Information Payload.

\begin{table}[h]
\centering
\small
\caption{Performances for different models on Long Range Arena. The performance of previous works in the first area are from ~\cite{tay2020long}.}
\vskip 0.15in
\setlength{\tabcolsep}{.5mm}
\begin{tabular}{@{}lccccccccc@{}}
\toprule
Model                              & \bf ListOps  & \bf Text  & \bf Retrieval  & \bf Image  & \bf Path.   & \bf Avg.       \\
\#Length                           & 2K           & 4K        & 4K             & 1K         & 1K               &                \\
\midrule
Transformer &  36.37 & 64.27 &
        57.46 &  42.44& 71.40  & 54.39\\

        Local Attention &  15.82 &52.98 & 53.39 & 41.46& 66.63  & 46.06 \\
        Sparse Trans.  & 17.07 & 63.58 & 59.59 & 44.24 & 71.71   & 51.24 \\
        Longformer& 35.63& 62.85 & 56.89&  42.22 & 69.71  & 53.46 \\
        Linformer &   35.70 & 53.94&  52.27 & 38.56 & 76.34  & 51.36\\
        Reformer &  37.27 & 56.10 & 53.40 &  38.07& 68.50  & 50.67 \\
        Sinkhorn Trans. &33.67 & 61.20 & 53.83 & 41.23 & 67.45  & 51.39\\
        Synthesizer & 36.99 & 61.68 & 54.67  &41.61 & 69.45  & 52.88\\
        BigBird  & 36.05 & 64.02 & 59.29  &  40.83 & 74.87  & 55.01 \\
        Linear Trans. & 16.13&  65.90 & 53.09 & 42.34 & 75.30  & 50.55 \\ 
        Performer  &18.01& 65.40 & 53.82 & 42.77 & 77.05  & 51.41\\
\midrule
Fnet  & 35.33        & 65.11     & 59.61         & 38.67      & 77.80           & 55.30              \\
H-Trans.-1D           & \textbf{49.53}           & \textbf{78.69}         & 63.99              & 46.05         & 68.78                & 61.41              \\
Nystromformer              & 37.15           & 65.52         & 79.56              & 41.58          & 70.94               & 58.95              \\
Luna-256      & 37.98           & 65.78     & 79.56          & 47.86      & \textbf{78.55}                & 61.95              \\
Pixelfly      & 37.65           & 66.78     & 80.55          & 42.35      & 72.01                & 59.86              \\
\midrule
\tabincell{l}{Hypercube\\ Trans.}                 & 37.48         & 63.79     & \textbf{81.16}          & \textbf{53.79}      & 74.12            & \textbf{62.07} \\
\bottomrule
\end{tabular}
\label{tab:lra-other}
\end{table}

\subsection{Long-Context Sequence Modeling}
To show the effectiveness of Hypercube Transformer, we present the performances of previous works on LRA briefly in Table \ref{tab:lra-other}. Compared to recent Fnet~\cite{leethorp2021fnet}, Nystromformer~\cite{xiong2021nystromformer}, LUNA~\cite{ma2021luna}, H-Transformer-1D~\cite{zhu2021hTransformer1d}, Pixelfly~\cite{chen2021pixelated}, our Hypercube Transformer achieves better average performance while using architecture with fewer parameters. H-Transformer-1D has very good performance on NLP tasks because their hierarchical method provides inductive bias for natural language. Hypercube Transformer improves the result of Image by a large margin because hypercube is like high-dimensional view of a picture and can catch more information. In summary, Hypercube Transformer shows potential for models only based on sparse graphs that can grab all interactions among tokens.

%%%%%%%%%%%%%%%%%%%%%%%%%%%%%%%%%%%%%%%%%%
%%%%%%%%%%%%%%%%%%%%%%%%%%%%%%%%%%%%%%%%%%
%%%%%%%%%%%%%%%%%%%%%%%%%%%%%%%%%%%%%%%%%%

\subsection{Masked Language Modeling for Large-Scale Pretraining}
One important application of Transformer is Large-Scale Pretraining like BERT~\cite{devlin2019bert}. Here we follow~\cite{devlin2019bert,izsak2021train} to pretrain Hypercube Transformer from scratch, denoted by CubeBERT, and finetune it on downstream tasks. We denote the original BERT-large by BERT, our pretrained BERT-large with 128 length by $\text{BERT}_{128}$, and our pretrained CubeBERT-large with 128 length by CubeBERT$_{128}$. The detailed structure of CubeBERT is the same as BERT and experimental details are provided in Appendix \ref{appendix:hyperparams}. We use English Wikipedia and BookCorpus2~\cite{gao2020pile} as our pretraining datesets. 
% We train CubeBERT$_{128}$ on 8$\times$3090 GPUs with batch size 4096 and 240k updates.

\begin{table}[!h]
\centering
\small
\caption{Finetuning MLM on Wikitext103.}
\vskip 0.15in
\begin{tabular}{@{}lccccccccc@{}}
\toprule
\bf Model     & \bf  Loss & \bf Speedup \\
\midrule
$\text{BERT}_{128}$  &    1.18    &   $1\times$  \\  %  80.98,81.55,80.94 & 53.91,53.73,53.74 \\% seed 45,53,54
$\text{CubeBERT}_{128}$ &     1.05     &  $1.4\times$ \\
\bottomrule
\end{tabular}
\label{tab:mlm}
\end{table}

\begin{figure}[h]
  \centering 
  \includegraphics[scale=0.28]{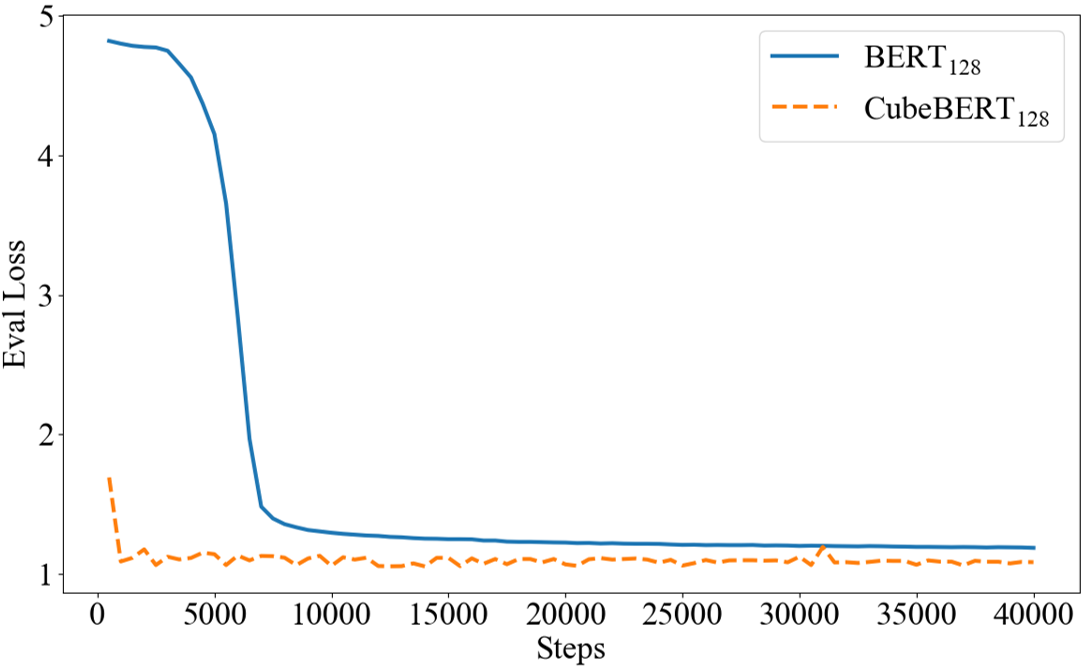} %
  \caption{CubeBERT$_{128}$ shows faster dropping rate of eval loss than $\text{BERT}_{128}$ when finetuning on Wikitext103.}
  \label{fig:mlm}
\end{figure}

\begin{table*}[h]
\centering
\small
\setlength{\tabcolsep}{.8mm}
\caption{Performances on GLUE test sets. For our implementation, results for RTE, STS and MRPC are reported by first finetuning on the MNLI model instead of the baseline pretrained model.}
\vskip 0.15in
\begin{tabular}{@{}lcccccccccc@{}}
\toprule
                                         & \bf MNLI-m/mm & \bf QNLI & \bf QQP & \bf RTE & \bf SST-2 & \bf MRPC & \bf CoLA          & \bf STS-B      & \bf Avg. & \bf Speedup \\
\#metric                                 & Acc           & Acc      & F1      & Acc     & Acc       & F1       & Matthew's corr.   & Spearman corr. & &  \\
\#Examples                               & 393k          & 105k     & 364k    & 2.5k    & 67k       & 3.7k     & 8.5k              & 7k             & &  \\
\midrule

BERT           & 86.0/85.2   & 92.6   & 72.0    & 78.3    & 94.5      & 89.9     & 60.9              & 87.5           & 83.0 & $1\times$\\
$\text{BERT}_{128}$             & 84.9/84.8     & 91.1     & 71.0    & 76.6   & 93.1      & 90.4    & 58.0              & 88.3           & 82.0 & $1\times$\\
\midrule
$\text{CubeBERT}_{128}$    & 85.9/85.0     & 90.8     & 71.3    & 77.1    & \textbf{95.3}      & 86.4     & \textbf{61.5}              & \textbf{87.6}           & 82.3 & $1.1\times$\\
\bottomrule
\end{tabular}
\label{tab:glue_tasks}
\end{table*}

\textbf{Finetuning on longer contexts.} We first finetune $\text{BERT}_{128}$ and CubeBERT$_{128}$ on a language model dataset Wikitext103 that could model sequences at 512 length to show the strong generalization ability of Hypercube Transformer. We initiate the position embeddings out of 128 randomly and finetune $\text{BERT}_{128}$ and CubeBERT$_{128}$ in a Masked Language Model (MLM) way, results are provided in Table \ref{tab:mlm}. We can see that the MLM loss (perplexity) of finetuned CubeBERT$_{128}$ (1.05) is lower than that (1.18) of our $\text{BERT}_{128}$, demonstrating the better generalization ability of Hypercube Transformer than vanilla one. At the same time, attribute to sparse attention, CubeBERT$_{128}$ still has a 1.4x speedup compared to $\text{BERT}_{128}$ at the training stage of finetuning. Another interesting finding is that the evaluation loss of CubeBERT$_{128}$ drops faster than $\text{BERT}_{128}$, implying that Hypercube Transformer could learn faster for training short and finetuning long in certain circumstances.

\textbf{Finetuning on GLUE.} For downstream tasks at 128 length, we finetune CubeBERT$_{128}$ on GLUE benchmark and compare CubeBERT$_{128}$ with BERT. Results are reported in Table \ref{tab:glue_tasks}. To do fair comparison, we also provide finetuning results for our reimplemented $\text{BERT}_{128}$. We observe that CubeBERT$_{128}$ achieves comparable performance without global attention routing information directly to [CLS] token. This demonstrate Hypercube sparsity is effective on information passing on graph. In detail, our CubeBERT$_{128}$ can have on par performance with BERT in most tasks (MNLI, QQP, SST-2, CoLA, STS-B), all with differences less than 1 point. For rest tasks like QNLI, RTE and MRPC, CubeBERT$_{128}$ is lower than BERT in 2 points except MRPC which has the smallest number of dataset examples. Overall, the average performance of CubeBERT$_{128}$ is slightly lower than BERT within 1 point and higher than $\text{BERT}_{128}$. The speedup for CubeBERT$_{128}$ is measured at the training training for finetuning GLUE. Because of the sequence length is 128, speedup is 1.1x, which is not remarkable compared to 1.4x for sequence at 512 length.

\section{Related Work}
To the best of our knowledge, our paper is related to sparse attention and their analysis.

\textbf{Theoretical Analysis.} Previous works mainly focus on the approximation of sparse attention to vanilla attention. BigBird~\cite{zaheer2021big} first proved  sparse attention mechanism defined by any graph containing star graph is a universal approximator. They also showed Turing Completeness of sparse encoder and sparse decoder. Pixelated Butterfly~\cite{chen2021pixelated} proved their flat butterfly matrices can approximate butterfly matrices which can tightly represent all structured matrices. Our paper focuses on finding graphs with better properties, not graphs to approximate complete graph. Axial Attention~\cite{ho2019axial} mentioned having the full receptive field, which is similar to grabbing all interactions among tokens in the paper. ~\cite{e21090863} proposed Information Capacity between two nodes based on the path transferring least information while we propose Information Payload based on all paths between two nodes.   

\textbf{Sparse Attention.} Previous Transformers adapt sparse attention including Star Transformer~\cite{guo2019starTransformer}, Sparse Transformer~\cite{child2019generating}, Longformer~\cite{beltagy2020longformer} and BigBird~\cite{zaheer2021big}. Compared to all those patterns, our proposed Hypercube Transformer adapts a fixed simple sparse pattern and is easy to implement. The recent flat butterfly pattern~\cite{chen2021pixelated} is also another simple sparse pattern with $O(N\log N)$ complexity with input length $N$. NAS has been applied to learning sparse patterns like SparseBERT~\cite{https://doi.org/10.48550/arxiv.2102.12871}, but searching methods cannot be applied to long-context tasks. Besides, their learned sparse patterns are data-dependent and cannot generalize.

\section{Conclusion}
We have introduced Normalized Information Payload (NIP), a graph scoring function for various graphs used in Transformer attention mechanism. By taking Computational Complexity and Information Payload into consideration, we can analyze sparse graphs via NIP to find what dense graph do we need for self-attention. After examining existed sparse patterns with NIP, we further present hypercube and utilize it in simple masked-based sparse Transformer. Hypercube Transformer achieves comparable or even better performances compared to strong baselines in pretrain tasks in NLP and long-context sequence modeling while reducing the usage of memory and computation. Experiments on different graphs on LRA benchmark also lay validation for Normalized Information Payload well. We hope our graph scoring function will reveal important parts behind different sparse patterns. In future work, we may utilize hypercube structure in other modules like MLPs. Another potential direction is to apply Hypercube Transformer to NLP tasks requiring long-context modeling like summarization and question answering.

\section*{Acknowledgements}
 This work was supported by the National Key Research and Development Program of China (No. 2020AAA0108702), the National Natural Science Foundation of China (No. 62022027) and the major key project of PCL (No. PCL2021A12).

\bibliography{references}
\bibliographystyle{icml2022}

\appendix
\onecolumn
\section{Computation for Normalized Information Payload}
\label{appendix:computation}
Since computation for NIP$(G)$ is related to CC$(G)$ and IP$(G)$, we respectively show these two parts.
\subsection{Computation for $\text{CC}(G)$}
We compute $\rho(G)$ and $\kappa(G)$ to get $\text{CC}(G)$. For all graphs, $\rho(G)$ can be computed easily through definition. For E-R random graph, we choose the probability of every edge to exist $p$ to be $\Theta\left(\frac{\log N}{N}\right)$ for the sake of connectivity. $\kappa(G)$ is the diameter of the graph, which can be straightforwardly computed by definition of graphs. For E-R random graph, the shortest path between any two nodes is logarithmic in the number of nodes~\cite{Chung15879,katzav2018distribution}, thus it equals to $\Theta(\log N)$. For $w$ used in ring lattice, we assume $w \ll N$ for approximation.

$\rho(G)$, $\kappa(G)$ and $\text{CC}(G)$ for graphs in Table \ref{tab:Gscore} are listed in Table \ref{tab:CCG}. 

\begin{table}[!ht]
\centering
\small

\caption{\footnotesize $\rho(G)$, $\kappa(G)$ and $\text{CC}(G)$ for graphs in Table \ref{tab:Gscore}, where $w$ is the number of neighbors of a ring lattice.}\label{tab:CCG}
\vskip 0.15in
\begin{tabular}{@{}lccccccccc@{}}
\toprule
Type of graph  & $\rho(G)$ & $\kappa(G)$     & $\text{CC}(G)$     \\
\midrule
Complete      & $\Theta(N)$             & $\Theta(1)$ & $\Theta(N)$                                            \\
% Ring lattice       & $\Theta(w)$              & $\Theta\left(\frac{N}{w}\right)$                          & $\Theta(N)$                                      \\
E-R random    &$\Theta(\log N)$     & $\Theta(\log N)$ & $\Theta(\log^2 N)$  \\
Tree          & $\Theta(1)$        & $\Theta(\log N)$                   & $\Theta(\log N)$                  \\
Star          & $\Theta(1)$              & $\Theta(1)$                                    & $\Theta(1)$                                                \\
\midrule

\tabincell{l}{Ring lattice\\\quad+ E-R random }   & $\Theta(\log N+w)$        & $\Theta(\log N)$                   &  $\Theta(\log N(\log N+w))$                      \\

\tabincell{l}{Ring lattice\\
\quad+ Star (Longformer) }         & $\Theta(w)$              & $\Theta(1)$                                       & $\Theta(w)$                                               \\
\tabincell{l}{Ring lattice\\
\quad+ Star\\
\quad+ E-R random (BigBird)}         & $\Theta(\log N+w)$              &$\Theta(1)$                                  & $\Theta(\log N+w)$                                             \\
\midrule
Hypercube     &$\Theta(\log N)$     & $\Theta(\log N)$         & $\Theta(\log^2 N)$         \\
\bottomrule
\end{tabular}
\end{table}

\subsection{Computation for $\text{IP}(G)$}
To computate $\text{IP}(G)$, we first find node pair $(a,b)$ whose distance is the diameter of the graph and calculate Information Payload for that node pair.

\begin{equation}
    I_{ab}=\sum_{P_{ab} \in  \mathcal{P}_{ab} }R(P_{ab}).
\end{equation}
For graphs in Table \ref{tab:Gscore}, $R(P_{ab})$ is constant for all paths $P_{ab} \in  \mathcal{P}_{ab}$. So we can compute IP$(G)$ as folows, where $|\mathcal{P}_{ab}|$ is the number of paths in $\mathcal{P}_{ab}$, 
\begin{equation}
    I_{ab}=|\mathcal{P}_{ab}|R(P_{ab}).
\end{equation}
For \textbf{Complete graph, Tree, Star graph, Ring lattice + random, Longformer pattern and BigBird pattern}, the number of paths $|\mathcal{P}_{ab}|$ is 1 and that path can be easily found. $R(P_{ab})$ for that path can be computed by definition. Here we choose the degree of non-global node in Longformer pattern and BigBird pattern to be $w$ and $\log N + w$ respectively.

\begin{table*}[!h]
\centering
\small
\caption{\footnotesize Information Payload for one path $R(P_{ab})$ and the number of paths $|\mathcal{P}_{ab}|$ between node $a$ and node $b$. $w$ is the number of neighbors of a ring lattice. $p$ is the probability of one edge exist in E-R random graph and is set as $\Theta(\frac{\log N}{N})$ for IP$(G)$.}
\vskip 0.15in
\begin{tabular}{@{}lccccccccc@{}}
\toprule
Type of graph  & $R(P_{ab})$ & $|\mathcal{P}_{ab}|$     & $\text{IP}(G)$     \\
\midrule
Complete      &   $\Theta\left(\frac{1}{N}\right)$           &             1                 & $\Theta\left(\frac{1}{N}\right)$                                              \\
% Ring lattice       &   $\Theta\left(\frac{1}{w^{\frac{N}{w}}}\right)$         &   $\Theta(N)$                      & $\Theta\left(\frac{N}{w^N}\right)$                                      \\
E-R random    &  $\Theta\left(\frac{1}{(\log N)^{\log N}}\right)$  &  $p^k(k-1)!C_{N-2}^{k-1}$ &  $\Theta\left(\frac{(N-2)!}{N^{\log N}(N-\log N)!}\right)$  \\
Tree          &   $\Theta(\frac{1}{N^{\log(9)}})$    &      1              & $\Theta(\frac{1}{N^{\log(9)}})$                   \\
Star          &     $\Theta\left(\frac{1}{N}\right)$        &     1                              & $\Theta\left(\frac{1}{N}\right)$                                               \\
\midrule

\tabincell{l}{Ring lattice\\\quad+ E-R random }   & $\Theta\left(\frac{1}{(\log N+w)^{\log N}}\right)$        &       $p^k(k-1)!C_{N-2}^{k-1}$           & $\Theta\left(\frac{(N-2)!}{(N+\frac{w}{\log N})^{\log N}(N-\log N)!}\right)$                   \\

\tabincell{l}{Ring lattice\\
\quad+ Star (Longformer) }         &  $\Theta\left(\frac{1}{Nw}\right)$        &  1                               & $\Theta\left(\frac{1}{Nw}\right)$                                              \\
\tabincell{l}{Ring lattice\\
\quad+ Star\\
\quad+ E-R random (BigBird)}         &    $\Theta\left(\frac{1}{N(\log N+w)}\right)$         &                                1 &$\Theta\left(\frac{1}{N(\log N+w)}\right)$                                               \\
\midrule
Hypercube     & $\Theta\left(\frac{1}{(\log N)^{\log N}}\right)$  &  $(\log N)!$        & $\Theta\left(\frac{(\log N)!}{(\log N)^{\log N}}\right)$           \\
\bottomrule
\end{tabular}
\end{table*}
For \textbf{E-R random graph and Ring lattice + E-R random}, we assume adding neighbors does not shorten the diameter of the graph. Thus the $|\mathcal{P}_{ab}|$ and length of the shortest path between any two nodes in ring lattice + E-R random is the same as those in E-R random graph. In E-R random graph, the expected number of a $k$-length path between two nodes is $p^k(k-1)!C_{N-2}^{k-1}$ where $p$ is the probability for one edge to exist. That is the value of $|\mathcal{P}_{ab}|$. The expected degree for every node in E-R random graph and ring lattice + E-R random is $\Theta(\log N)$ and $\Theta(\log N +w)$ respectively, so the expectation of $R$ for one path is $\Theta(\frac{1}{(\log N)^{\log N}})$ and $\Theta\left(\frac{1}{(\log N+w)^{\log N}}\right)$ respectively. In this case, $k=\log N$ and $p=\Theta(\frac{\log N}{N})$, we compute $\text{IP}(G)$ for E-R random graph as follows for example.
\begin{align}
        \text{IP}(G) &=p^k(k-1)!C_{N-2}^{k-1} \times \Theta\left(\frac{1}{(\log N)^{\log N}}\right) \\
        &=\Theta\left((\frac{\log N}{N})^{\log N}(\log N-1)!C_{N-2}^{\log N-1}\frac{1}{(\log N)^{\log N}}\right)\\
        &=\Theta\left(\frac{(\log N-1)!C_{N-2}^{\log N-1}}{N^{\log N}}\right)\\
        &=\Theta\left(\frac{(N-2)!}{N^{\log N}(N-\log N)!}\right).
\end{align}

For \textbf{Hypercube}, every node in $P_{ab}$ has a degree of $\Theta(\log N)$ and $len(P_{ab})$ equals to $\Theta(\log N)$, so $R(P_{ab})$ is $\Theta\left(\frac{1}{(\log N)^{\log N}}\right)$. For $|\mathcal{P}_{ab}|$, we consider the choices of every step in the random walk. For the first step it is $\log N$, and it is $\log N -1$ for the second step. Thus the $|\mathcal{P}_{ab}|$ equals to the full-permutation number of $\log N$, $(\log N)!$.

\section{Proofs and Algorithm}
\label{appendix:proof}

\subsection{Proof for Theorem \ref{Iab}}
\label{proof:Iab}

Theorem \ref{Iab} states that Information Payload between two nodes $I_{ab}$ equals to the probability of a random walk starts from node $b$ that ends in node $a$ at step $len(P_{ab})$.

\begin{proof}
First we introduce lemma \ref{pabe}.
\begin{lemma}
\label{pabe}
Information Payload for one path $R(P_{ab})$ equals to the probability ($Pr(P_{ba})$) of a random walk starts from $b$ and ends in $a$ following reversed path of $P_{ab}$.
\end{lemma}

\begin{proof} 
Let us consider random walk on the reversed path of $P_{ab}$, namely $P_{ba}$. For the $i$th step we take, the probability equals to $1/deg(v)$ where $v$ is the node that we are at the $i-1$th step. So the total probability of this path is $Pr(P_{ba})=\prod_{v \in \mathcal{P}_{ba} \And v \neq a}\frac{1}{deg(v)}$. We know that $\{v | \in P_{ba} \And v \neq a\}$ equals to $\{v | v \in P_{ab} \And v \neq a\}$. So $R(P_{ab})=Pr(P_{ba})$.
\end{proof}

Then, the probability of a random walk starts from $b$ that ends in $a$ at the $len(P_{ab})$ step, denoted by $SPr(P_{ba})$ is the summation of the probability of all paths in $P_{ba}$.
\begin{equation}
    SPr_{ba}=\sum_{P_{ba} \in \mathcal{P}_{ba}} Pr(P_{ba}).
\end{equation}

For every $R(P_{ab})$, from lemma \ref{pabe} we know that for every $P_{ab} \in \mathcal{P}_{ab}$, $R(P_{ab})=Pr(P_{ba})$. So the Information Payload between two nodes $a,b$
\begin{align}
    I_{ab} &=\sum_{P_{ab} \in \mathcal{P}_{ab}}R(P_{ab}) \\
    &=\sum_{P_{ab} \in \mathcal{P}_{ab}}Pr(P_{ba})\\
    &=\sum_{P_{ba} \in \mathcal{P}_{ba}}Pr(P_{ba})\\
    &=SPr_{ba}.
\end{align}
\end{proof}
Since we have Theorem \ref{Iab}, we can compute $\text{IP}(G)$ via random walk as below. 

\textbf{Computing Information Payload $\text{IP}(G)$ via random walk.} Using adjacent matrix $A_G$ and diagonal matrix $D=diag(\frac{1}{d_1},\frac{1}{d_2},...\frac{1}{d_N})$, we can easily calculate $I_{ab}$ by random walk.

For each $a,b$, we have
\begin{equation}
\label{randomwalk}
\begin{split}
M      & = DA_G, \\
M^i    & = M^{len(P_{ab})}, \\
I_{ab}      & = [(M^i)^T]_{ab}.
\end{split}
\end{equation}

%where $I_{ab}$ is the values in matrix $I \in R^{N \times N}$ at row $a$ and column $b$.

According to Definition \ref{def:FG}, let $len(P_{ab})=\kappa(G)$ in equation \eqref{randomwalk} and $\Delta$ be the set of node pairs whose distance is the diameter of the graph, we can get 
\begin{equation}
\text{IP}(G):=\min_{(a,b) \in \Delta}([((DA_G)^{\kappa(G)})^T]_{ab}).
\end{equation}

\subsection{Proof for Theorem \ref{blocksparse}}
\label{proof:blocksparse}

Theorem \ref{blocksparse} states that larger block size makes star graph and hypercube have less Normalized Information Payload.

\begin{proof}
Given one graph $G_0(\mathcal{V}_0,\mathcal{E}_0)$ where $|\mathcal{V}_0|=N_0$, we denote the two graphs adapting block sparse with different block sizes $x$ and $y$ by $G_{x}(\mathcal{V}_x,\mathcal{E}_x)$ and $G_{y}(\mathcal{V}_y,\mathcal{E}_y)$. We know that $|\mathcal{V}_x|$ and $|\mathcal{V}_y|$ all equal to $N_0$. Let $x<y$, our goal is to prove that NIP$(G_x)>\text{NIP}(G_y)$.

We first use $b$ to denote the block size and deduce NIP$(G_b)$. We have another affiliated graph $G_{1/b}(\mathcal{V}_{1/b},\mathcal{V}_{1/b})$ where $|\mathcal{V}_{1/b}|=\frac{N_0}{b}$ that adapts the same sparse pattern as $G_0$. Here we use a function of sequence length $N$ to denote NIP$(G)$, namely $NIP_G(N)$.

We first consider CC$(G_b)$.
\begin{align}
    \kappa(G_b)&=\kappa(G_{1/b}), \\
    \rho(G_b)&=b\rho(G_{1/b}),\\
\end{align}
So 
\begin{equation}
        \text{CC}(G_b)=b \cdot \text{CC}(G_{1/b}).
\end{equation}

Next, for IP$(G_b)$, the longest path in $G_b$ is equal to that in $G_{1/b}$, while any node $v$ in the path changes its degree $deg(v)$ to $b\cdot deg(v)$. Thus we have 

\begin{equation}
      \text{IP}(G_{b})=\frac{\text{IP}(G_{1/b})}{b^{\kappa(G_{1/b})}}. 
\end{equation}

The final $\text{NIP}(G_b)$ then 
\begin{align}
    \text{NIP}(G_b)&=\frac{\text{IP}(G_{b})}{\text{CC}(G_b)} \\
    &=\frac{\text{IP}(G_{1/b})}{b \cdot \text{CC}(G_{1/b})b^{\kappa(G_{1/b})}} \\
    &=\frac{1}{b^{\kappa(G_{1/b})+1}}\text{NIP}(G_{1/b})\\
    &=\frac{1}{b^{\kappa(G_{1/b})+1}}NIP_G\left(\frac{N_0}{b}\right).
\end{align}
Let $f(b)=\text{NIP}(G_b)$, our goal is to prove that $f(b)$ is monotonically decreasing for star, hypercube.

\textbf{Star graph}. While $\kappa(G)$ for star is always 2, the $f(b)$ for star is as follows
\begin{align}
    f(b)&=\frac{1}{b^3}NIP_G\left(\frac{N_0}{b}\right)\\
    &=\frac{b}{N_0b^3} \\
    &=\frac{1}{N_0b^2}. 
\end{align}
It's monotonically decreasing for $b$.

\textbf{Hypercube}. We do the same computation for hypercube. $\kappa(G)$ for hypercube is $\log N$, so 
\begin{align}
    f(b) &=\frac{1}{b^{\log N_0 +1}}NIP_G\left(\frac{N_0}{b}\right)\\
    &=\frac{1}{b^{\log N_0 +1}}\frac{(\log \frac{N_0}{b})!}{(\log \frac{N_0}{b})^{\log \frac{N_0}{b}+2}}.
\end{align}
Using Sterling Equation to approximate factorial, where $c$ is $\sqrt{2\pi}$, we get 
\begin{align}
    f(b) &\approx \frac{1}{b^{\log N_0 +1}}\frac{c(\log \frac{N_0}{b})^{(\log \frac{N_0}{b})+\frac{1}{2}}e^{-(\log \frac{N_0}{b})}}{(\log \frac{N_0}{b})^{\log \frac{N_0}{b}+2}}\\
    &=\frac{1}{b^{\log N_0 +1}}\frac{c}{(\log \frac{N_0}{b})^{1.5}e^{(\log \frac{N_0}{b})}}\\
    &= \frac{1}{b^{\log N_0 +1}}\frac{c}{(\log \frac{N_0}{b})^{1.5} (\frac{N_0}{b})^{\frac{1}{\ln 2}}}\\
    &=\frac{C}{b^{\log N_0 +1 -\frac{1}{\ln 2}}(\log \frac{N_0}{b})^{1.5} }.
\end{align}
while $C$ is another constant. To make $f(b)$ monotonically decreasing, the derivative of $f(b)$ should be less than zero, which is true when
\begin{equation}
    \log N_0 +1 -\frac{1}{\ln 2} > \frac{3}{2\ln 2(\log N_0 - \log b)}.
\end{equation}
Note that $b \leq \frac{N}{2}$, thus we have to prove
\begin{equation}
    \log N_0 +1 -\frac{1}{\ln 2} > \frac{3}{2\ln 2},
\end{equation}
which is true for $N_0 \geq 128$.

Now we have proved Theorem \ref{blocksparse}.

\end{proof}

\subsection{Iterative mapping algorithm}
\label{alg:mapping}
We demonstrate the iterative binary mapping algorithm here. $<<$ means left logical shift for binary numbers.
\begin{algorithm}[]
  \caption{Binary representation of sequences }
  \label{Br}
\begin{algorithmic}
  \STATE {\bfseries Input:} sequence $S=(s_0,...,s_{N-1})$
  \STATE {\bfseries Output:} Binary representation  $X^N=X_0X_1X_2...X_i...X_{N-1}$

  \STATE Initialize $X=(0,1)$
    \REPEAT 
    \STATE $Y=()$
    
  \FOR{$i$ {\bfseries in} $X$}
    \STATE   $ Y.append(i<<1)$
  \ENDFOR
  \FOR{$i$ {\bfseries in} reversed $X$}
    \STATE   $ Y.append(i << 1+1)$
  \ENDFOR
  \STATE $X=Y$

  \UNTIL{$len(X)>=N$}
   
    \STATE {\bfseries Output:} $X^N = X[:N]$
\end{algorithmic}
\end{algorithm}

Final representation 
\begin{equation}
\label{eq:Xi}
     X_i^d = \left(\lfloor \frac{i \mod 2^{k-d+1}} {2^{k-d}}\rfloor + \lfloor \frac{i \mod 2^{k-d}}{2^{k-d-1}}\rfloor\right)\mod 2
\end{equation}

can be proved by mathematical induction .

\begin{proof}
\begin{item}
\item Base case: For $k=1$, we know that equation \eqref{eq:Xi} is true.
\item Inductive step: Assume for $k=n-1$ equation \eqref{eq:Xi} is true we deduce it for $k=n$.
If $d = 0$, we can easily verify that equation \eqref{eq:Xi} is true. 

For $d \neq 0$ situations, if $i < 2^{n-1}$, according to algorithm \ref{Br}, $X_i^d (k=n) = X_i^{d-1} (k=n-1) $, so 
\begin{align}
    X_i^d (k=n) &= X_i^{d-1} (k=n-1) \\
    &=\left(\lfloor \frac{i \mod 2^{n-1-(d-1)+1}} {2^{n-1-(d-1)}}\rfloor + \lfloor \frac{i \mod 2^{n-1-(d-1)}}{2^{n-1-(d-1)-1}}\rfloor\right)\mod 2 \\
    &=\left(\lfloor \frac{i \mod 2^{n-d+1}} {2^{n-d}}\rfloor + \lfloor \frac{i \mod 2^{n-d}}{2^{n-d-1}}\rfloor\right)\mod 2.
\end{align}
If $i \geq 2^{n-1}$, $X_i^d (k=n) = X_{2^n-1-i}^d (k=n) = X_{2^n-1-i}^{d-1} (k=n-1) $, so 
\begin{align}
    X_i^d (k=n) &= X_{2^n-1-i}^{d-1} (k=n-1) \\
    &=\left(\lfloor \frac{(2^n-1-i) \mod 2^{n-1-(d-1)+1}} {2^{n-1-(d-1)}}\rfloor + \lfloor \frac{(2^n-1-i) \mod 2^{n-1-(d-1)}}{2^{n-1-(d-1)-1}}\rfloor\right)\mod 2 \\
    &=\left(\lfloor \frac{(2^n-1-i) \mod 2^{n-d+1}} {2^{n-d}}\rfloor + \lfloor \frac{(2^n-1-i) \mod 2^{n-d}}{2^{n-d-1}}\rfloor\right)\mod 2 \\
    &=\left(\lfloor \frac{i \mod 2^{n-d+1}} {2^{n-d}}\rfloor + \lfloor \frac{i \mod 2^{n-d}}{2^{n-d-1}}\rfloor\right)\mod 2.
\end{align}

\end{item}
\end{proof}

\section{Hyper-parameters}
\label{appendix:hyperparams}
\subsection{Long Range Arena}
\label{appendix:lra}

We put hyper-parameters for LRA here. We set the embedding hidden size to 64 and the hidden size for attention to be 128. Dropout rate and weight decay is different for each task.

\begin{table}[!h]
\centering
\small
\caption{Hyper-parameters used for all models we trained on Long Range Arena.}
\vskip 0.15in
\begin{tabular}{@{}lccccccccc@{}}
\toprule
\bf Hyper-parameter          &  \bf   Our Model          \\
\midrule
Batch size & 32           \\
Number of Layers        & 4              \\
Number of Shared Layers & 2              \\
Hidden size             & 64             \\
FFN inner hidden size   & 128            \\
Attention heads         & 4              \\
Attention head size     & 32             \\
Block size            & 16             \\
Dropout                 & 0.1, 0.2, 0.3  \\
Attention Dropout       & 0             \\
Learning Rate Decay     & Cosine         \\
Weight Decay            & 0, 0.0001      \\
Optimizer               & AdamW          \\
Adam $\epsilon$         & 1e-6           \\
Adam $\beta_1$          & 0.9            \\
Adam $\beta_2$          & 0.98           \\
Gradient Clipping       & 0             \\
Prediction Head Pooling & mean           \\
\bottomrule
\end{tabular}
\label{tab:lra_hparams}
\end{table}

\begin{table}[!h]
\centering
\small
\caption{Sparsity settings. $b$ is the block size.}
\vskip 0.15in
\begin{tabular}{@{}lccccccccc@{}}
\toprule
\bf Graph                             &\bf Global  tokens      & \bf Window length &\bf Random tokens      &\bf Blocks (1K)       &\bf Blocks (2K)     &\bf Blocks (4K)      \\
\midrule
Star                            &  $1 \times b $      & 0             & 0                  & 190               & 382             & 766              \\
\tabincell{l}{Ring lattice\\
\quad+ E-R random}                    &  0                  & $3 \times b $ & $5 \times b $      & 498               & 1006            & 2028             \\
\tabincell{l}{Ring lattice\\
\quad+ Star (Longformer) }        &  $1 \times b $      & $3 \times b $ & 0                  & 314               & 634             & 1274             \\
\tabincell{l}{Ring lattice\\
\quad+ Star\\
\quad+ E-R random (BigBird)}  &  $1 \times b $      & $3 \times b $ & $4 \times b $      & 546               & 1119            & 2274             \\
Hypercube                         &  0                  & $\_$          & $\_$               & 448               & 1024            & 2304             \\
\bottomrule
\end{tabular}
\label{tab:blockxxx}
\end{table}

\subsection{CubeBERT$_{128}$ Hyper-parameters}
We present hyper-parameters for pretraining in Table \ref{tab:pretrain_params} and downstream tasks in Table \ref{tab:glue_params}.

\begin{table}[!h]
\centering
\small
\caption{Hyper-parameters used for CubeBERT$_{128}$ pretraining.}
\vskip 0.15in
\begin{tabular}{@{}lccccccccc@{}}
\toprule
\bf Hyperparameter  & \bf Our Model \\
\midrule 
Number of Layers & 24 \\
Hidden size & 1024 \\
FFN inner hidden size & 4096 \\
Attention heads & 16 \\
Attention head size & 64 \\
Dropout & 0.1 \\
Attention Dropout & 0.1 \\
Learning Rate Decay & Linear \\
Weight Decay & 0.01 \\
Optimizer & AdamW \\
Adam $\epsilon$ & 1e-6 \\
Adam $\beta_1$ & 0.9 \\
Adam $\beta_2$ & 0.98 \\
Gradient Clipping & 0.0 \\ \midrule
Batch Size & 4096 \\
Peak Learning Rate &  1e-3 \\
Warmup Proportion &  2\% \\
Max Steps & 240k  \\
\bottomrule
\end{tabular}
\label{tab:pretrain_params}
\end{table}

\begin{table}[!h]
\centering
\small
\caption{Hyper-parameters used for downstream tasks,}
\vskip 0.15in
\begin{tabular}{@{}lccccccccc@{}}
\toprule
\bf Hyper-parameter          & \bf   GLUE &\bf Wikitext103           \\
\midrule
Batch Size        & \{32,16\}        &      8 \\
Learning Rate             & \{2e-5, 5e-5\}          &5e-5 \\
Weight Decay & 0.01& 0 \\
Max Epochs & 5 & 3 \\
Warmup Steps & 50 & 0 \\
\bottomrule
\end{tabular}
\label{tab:glue_params}
\end{table}

\end{document}